\documentclass[letterpaper, 10 pt, conference]{ieeeconf}

\IEEEoverridecommandlockouts
\usepackage{graphicx}
\graphicspath{{figs/}}
\usepackage{amsmath,amssymb}
\usepackage{tikz-cd}
\usepackage{stmaryrd}
\usepackage{xcolor}
\usepackage{csquotes}
\usepackage{tikz}
\usepackage{caption}
\usepackage{subcaption}
\usepackage{xcolor}
\usepackage{booktabs}
\usepackage{microtype}
\usepackage{cite}
\usepackage{algorithm}
\usepackage{algpseudocode}
\usepackage{subcaption} 
\usepackage{comment}
\usepackage{cite}

\newtheorem{theorem}{Theorem}

\newtheorem{lemma}{Lemma}
\newtheorem{proposition}{Proposition}
\newtheorem{problem}{Problem}

\renewcommand{\leq}{\leqslant}
\renewcommand{\geq}{\geqslant}
\renewcommand{\phi}{\varphi}
\renewcommand{\epsilon}{\varepsilon}

\newcommand{\M}{\mathcal{M}}

\newcommand{\E}{\mathbb{E}}
\newcommand{\A}{\mathcal{A}}
\newcommand{\B}{\mathcal{B}}

\renewcommand{\P}{\mathbb{P}}
\newcommand{\R}{\mathbb{R}}
\newcommand{\Q}{\mathbb{Q}}

\newcommand{\D}{\mathcal{D}}
\newcommand{\btheta}{\boldsymbol{\theta}}

\usepackage{todonotes}

\begin{document}

\title{\bf Distributionally Robust Clustered Federated Learning:\\A Case Study in Healthcare}

 \author{Xenia Konti$^{1}$ \quad Hans Riess$^{2}$ \quad Manos Giannopoulos$^{1}$ \quad  Yi Shen$^{3}$ \\ Michael J.~Pencina$^{4}$ \quad Nicoleta J~Economou-Zavlanos$^{4}$ \quad Michael M.~Zavlanos$^{1,2,3}$
\thanks{$^{1}$Dept.~of Computer Science, Duke University.}%
\thanks{$^{2}$Dept.~of Electrical and Computer Engineering, Duke University.}%
\thanks{$^{3}$Dept.~of Mechanical Engineering \& Materials Science, Duke University.}%
\thanks{$^{4}$School of Medicine, Duke University.}%
\thanks{Emails: \{pek7, hmr14, eg261, ys267, mjp50, nje6, mz61\}@duke.edu}%
\thanks{This work is supported in part by The Duke Endowment (TDE) under grant \#7262-SP and by the Onassis Foundation under award \#ZT033-1/2023-2024.}%
}

\maketitle

\begin{abstract}
In this paper, we address the challenge of heterogeneous data distributions in cross-silo federated learning by introducing a novel algorithm, which we term Cross-silo Robust Clustered Federated Learning (CS-RCFL). Our approach leverages the Wasserstein distance to construct ambiguity sets around each client's empirical distribution that capture possible distribution shifts in the local data, enabling evaluation of worst-case model performance. We then propose a model-agnostic integer fractional program to determine the optimal distributionally robust clustering of clients into coalitions so that possible biases in the local models caused by statistically heterogeneous client datasets are avoided, and analyze our method for linear and logistic regression models. Finally, we discuss a federated learning protocol that ensures the privacy of client distributions, a critical consideration, for instance, when clients are healthcare institutions. We evaluate our algorithm on synthetic and real-world healthcare data.
\end{abstract}

\section{Introduction}
\label{sec:intro}

The complex nature of healthcare systems and data presents significant regulatory \cite{gostin2001} and ethical \cite{price2019} challenges associated with the design and deployment of large-scale machine learning models that need to comply with, e.g., the Health Insurance Portability and Accountability Act (HIPAA) in the United States or similar laws in other nations, which generally prohibit the sharing of patient data across healthcare organizations. At the cutting-edge, are distributed and privacy-preserving machine learning systems, broadly known as \emph{Federated Learning} (FL)  systems \cite{kairouz2021}, which tackle the regulatory problem of data-sharing by avoiding it altogether; they exploit accumulated statistical information on relevant patient populations or machine learning model parameters that carry no patient-specific information.

It is well known that the design of effective machine learning algorithms requires the availability of sufficient data. Depending on the characteristics of the model and the dimension of the feature space, a theoretical minimum number of samples, called the sample complexity, is required for a model to generalize to unseen data sampled from the same distribution on which it was trained \cite{vapnik2013nature}. In practice, however, such data are often unavailable. Federated learning algorithms overcome this challenge by instantiating a cross-silo collaboration across data centers, such as those housed within healthcare organizations, that allows sharing of data insights rather than raw data and uses this distributed knowledge to train better-performing models. These data centers and the organizations they are housed in are often referred to as \emph{clients}; in healthcare systems they are simply the \emph{hospitals}. Among the many possibilities, FL has promising applications in healthcare informatics \cite{xu2021} and smart healthcare technologies such as remote monitoring, disease detection, and medical imaging; see \cite{nguyen2022b} for a survey. 

An important challenge in the design of effective federated learning algorithms is that the participating organizations are often subject to statistically heterogeneous datasets. This is, e.g., the case in healthcare applications, where hospitals often attend to statistically different patient populations. In an effort to ensure model fairness, meaning that the learned local models are free of any possible biases that may arise by aggregating models trained on different data distributions, \emph{personalized federated learning} algorithms have been developed \cite{tan2022}, that take advantage of collaboration between hospitals while at the same time implementing ways to mitigate bias. In healthcare applications, bias can be caused by demographic, geographic, financial, educational, political, and environmental differences across different healthcare organizations and the patient populations they serve \cite{penman2016}. A common approach to personalized federated learning is, what is known as, \emph{Clustered Federated Learning} (CFL)  \cite{cho2021personalized} that reduces bias by clustering together and restricting collaboration between hospitals that potentially have a common or similar data-generating distributions.


In this work, we propose a new clustered federated learning algorithm that accounts for both statistical heterogeneity in the data of different participating organizations and statistical sampling uncertainty (e.g., due to finite sampling) or other distributions shifts in the local data of each individual organization. For this, we employ ideas from Distributionally Robust Optimization (DRO) to construct an integer fractional program that we solve to determine the optimal
distributionally robust clustering of clients into coalitions so that (i) the learned local models generalize to unseen data sampled from distributions that are different from, but close enough to, the training ones and (ii) possible biases in the local models caused by statistically
heterogeneous hospital datasets are avoided. We term our proposed algorithm Cross-Silo Robust Clustered Federated Learning (CS-RCFL). We theoretically justify our algorithm for two different types of models: linear regression with absolute loss and logistic regression. We finally evaluate our algorithm on both synthetic and real-world healthcare data, and show that it outperforms related methods.

\subsection{Related Work}

\subsubsection{Clustered Federated Learning (CFL)}

CFL methods often construct clusters in rounds, potentially training models for several epochs between consecutive rounds. This approach was employed in \cite{sattler2020} that iteratively clusters users by bi-partitioning existing clusters at each round. To address the high communication overhead of typical CFL methods \cite{shlezinger2020}, more efficient algorithms have been recently developed. An example is IFCA \cite{ghosh2020}, which randomly initializes cluster centers and assigns clients to the nearest one, though this method is sensitive to center initialization. Similarly, FeSEM \cite{long2023} relies on an $\ell_2$-distance -based expectation maximization (EM) algorithm. The convergence of these methods has been studied in \cite{ma2022}. There are also works that deviate from model similarity-based clustering, introducing alternative coalition formation objectives \cite{luo2022, mansour2020, morafah2023} or incorporating robustness notions to produce robust clustering schemes \cite{duan2021, werner2023, guo2023}. The literature discussed above focuses on cross-device federated learning, involving personal devices such as smartphones. Instead, in this paper, we focus on cross-silo federated learning that involves data centers affiliated with organizations \cite{huang2022}.
In cross-silo FL, global models are typically trained iteratively with clients sharing locally fine-tuned versions. Clustered cross-silo federated learning methods construct coalitions during each round using model similarity measures \cite{huang2021, jiang2022} or game-theoretic incentive compatibility \cite{bao2023}. Common in the methods discussed above is that they do not account for possible distribution shifts between the training and test data. In this paper, we address this challenge employing notions from DRO.

\subsubsection{DRO in Federated Learning}
Interest in robust federated learning from a Distributionally Robust Optimization (DRO) perspective has grown recently. Agnostic FL \cite{mohri2019} was the first framework to optimize over the worst-case mixture of local models, instead of assigning uniform or size-proportional weights. An improved algorithm with less communication overhead was introduced in \cite{deng2020}, which uses periodic averaging and samples a subset of clients at each round. By relaxing the worst-case mixture of local models, \cite{pillutla2023} imposes constraints on the weights determined by the superquantile. Similar to our approach, but for non-clustered FL, \cite{nguyen2022, reisizadeh2020} consider distribution shifts and rely on Wasserstein-DRO to account for them. Here we extend this framework to clustered FL by designing client coalitions that are robust to possible uncertainties in the training data distributions.

\section{Background}
\label{sec:background}

\subsection{Federated Learning}

Federated learning (FL) is a privacy-centric distributed learning framework. Its appeal is mostly attributed to its scalable and decentralized nature, allowing for simultaneous training of local models on local datasets that will, then, be aggregated into a global model that combines the predictive capabilities of its individual parts. At the same time, federated learning guards individual users' data by providing privacy guarantees \cite{bonawitz2019}. In the standard terminology, several \emph{clients} (e.g.,~hospitals) and one or more \emph{servers} can form a FL system. Communication takes place between the clients and the server (vertical FL) and possibly between the clients directly (horizontal FL). For instance, clients may upload parameters to the server or download parameters or aggregations of parameters that the server has received. As discussed before, FL frameworks are also classified by the types of clients, including mobile devices (cross-device FL) or data silos managed by organizations (cross-silo FL). In this paper, we develop a variation of the well-known FedAve \cite{mcmahan2017} algorithm, which aggregates individual models into global models by taking weighted averages of model parameters; see Eq.~\eqref{eq:agg_theta}. 

\subsection{Distributionally Robust Optimization (DRO)}

We first recall some technical definitions from optimal transport theory \cite{thorpe2018introduction}  in order to define distances between distributions, essential to our problem.
Suppose $\Q_1, \Q_2$ are distributions supported on a subset $\Xi$ of a euclidean space with a chosen metric $\delta$ (e.g., the $\ell_2$-norm), and suppose $\Gamma(\Q_1,\Q_2)$ is the set of joint distributions on $\Xi \times \Xi$ with marginals $\Q_1$ and $\Q_2$. Then, the $1$-\emph{Wasserstein distance} \cite{ruschendorf1985wasserstein} between $\Q_1$ and $\Q_2$ is defined as
\begin{align*}
    W_1(\Q_1, \Q_2) =  \textstyle\inf_{\P \in \Gamma(\Q_1,\Q_2)} \biggl\{ \int_{\Xi \times \Xi} \delta(\xi_1,\xi_2) d \P(\xi_1, \xi_2) \biggr\}.
\end{align*}

We consider the problem of minimizing an objective function under uncertainty. Suppose $f: \R^p \times \Xi \to \R$ is an objective function with deterministic decision variables and uncertain parameters $x \in \R^p$ and $\xi \in \Xi$, respectively, and suppose $\hat{\Q}$ is an estimated distribution of $ \xi \in \Xi$. Where robustness to distributional shifts is desired, a common practice is to form \emph{ambiguity sets}, that is sets $\mathcal{P}$ of probability distributions, that effectively convert the objective function $f$ to an objective function of the form $\sup_{\Q \in \mathcal{P}} \E_{\xi \sim \Q}\bigl[f(x,\xi)\bigr]$. 
A typical choice for $\mathcal{P}$ is a ball of radius $\epsilon >0$ centered around an estimated distribution $\hat{\Q}$ in the metric space of probability measures with the $W_1$ distance metric, denoted by $\B_\epsilon(\hat{\Q}) = \{ \Q : W_1(\Q, \hat{\Q}) \leq \epsilon \}$, giving rise to the robust optimization problem
\begin{align}
\label{eq:DRO-main}
    \min_{x} \textstyle\sup_{\Q \in \B_\epsilon(\hat{\Q})} \E_{\xi \sim \Q}\bigl[f(x,\xi)\bigr].
\end{align}
While solving \eqref{eq:DRO-main} is generically computationally expensive, a dual formulation can lead to tractable solutions.

\begin{lemma}[Strong Duality \cite{esfahani2017}] \label{lem:dual}
Suppose $\Xi \subseteq \R^d$ and $f: \R^m \times \Xi \to \R$ is proper, convex, and lower semi-continuous. Suppose $\hat{\Q} \in \M(\Xi)$ is an estimated distribution on $\Xi$. Consider the worst-case expectation problem 
\begin{align}
   \textstyle\sup_{\Q \in \B_{\epsilon}(\hat{\Q})} \mathbb{E}_{\xi \sim \Q}\left[f(\xi)\right]. \label{eq:primal}
\end{align}
Then, the dual formulation
\begin{align}
    \label{eq:dual-formulation}
    \inf_{\lambda \geq 0} \Bigl\{\lambda \epsilon + \textstyle\E_{\xi' \sim \hat{\Q}}\bigl[\sup_{\xi \in \Xi} (f(\xi) - \lambda \delta(\xi', \xi) )\bigr]\Bigr\}
\end{align}
has a zero duality-gap.
\end{lemma}

Thus, we can solve \eqref{eq:dual-formulation} to get a solution for \eqref{eq:DRO-main}.
Finally, we recall the dual formulation for both linear regression with absolute loss \cite{chen2018} and logistic regression \cite{shafieezadeh2015distributionally}.

\begin{lemma}[Linear regression \cite{chen2018}] \label{lem:linear} 
    Suppose $l_{\theta}(x,y) = \left|y - x^\top \theta\right|$. Then, $\inf_{\theta}\sup_{\Q \in \B_{\epsilon}(\hat{\P}_N)} \E_{(x,y)\sim \Q}[l_{\theta}(x,y)]$ is equal the solution of the following optimization problem:
\begin{align}
\label{eq:lemma-2}
    \begin{aligned}
    \min_{\theta, \alpha, b}  &\quad \alpha\epsilon + \frac{1}{N}\textstyle\textstyle\sum_{i=1}^{N} b_i  \\
    \text{s.t.} &\quad \lVert \theta \rVert_2^2 + 1 \leq \alpha^2  \\
    &\quad y_i - x_i^\top \theta \leq b_i, \forall i \in N \\
    &\quad -(y_i - x_i^\top \theta) \leq b_i, \forall i \in N \\
    &\quad \alpha, b_i \geq 0, \forall i \in N \\
    &\quad \theta \in \B.
    \end{aligned}
\end{align}
\end{lemma}

\begin{lemma} [Logistic regression \cite{shafieezadeh2015distributionally}]
\label{lem:logistic}
Suppose $l_{\theta}(x,y) = \log (1+ \exp(-y \cdot x^T\theta))$ and define a distance metric between two data points ($x,y$) and ($x^\prime, y^\prime$) as $d((x, y), (x^\prime, y^\prime)) = ||x - x^\prime|| +\kappa\frac{|y - y^\prime|}{2}$, where $\kappa$ is a positive weight. Then, $\inf_{\theta}\sup_{\Q \in \B_{\epsilon}(\hat{\P}_N)} \E_{(x,y)\sim \Q}[l_{\theta}(x,y)]$ is equal to the solution of the following optimization problem:
\begin{align}
    \begin{aligned}
    \label{eq:lemma-3}
    \min_{\theta, \alpha, b}  &\quad \alpha\epsilon + \frac{1}{N}\textstyle\textstyle\sum_{i=1}^N b_i  \\
    \text{s.t.} &\quad l_{\theta}(x_i,y_i) \leq b_i, \forall i \in N \\
    &\quad l_{\theta}(x_i,-y_i) - \alpha \kappa \leq b_i, \forall i \in N \\
    &\quad ||\theta||_\star \leq \alpha, \forall i \in N \\
    &\quad \alpha, b_i \geq 0, \forall i \in N \\
    &\quad \theta \in \B.
    \end{aligned}
\end{align}
\end{lemma}


\vspace{0.2cm}

\section{Problem Formulation}
\label{sec:problem}

Consider $N$ hospitals (the clients in FL) and assume that each hospital $i\in \{1,\dots,N\}$ has data collected in the set $\D_i = \big\{(x_i^{(p)},y_i^{(p)})\big\}_{p=1}^{|\D_i|}$, where every patient $p \in \{1,\dots,|\D_i|\}$ in the sample population of hospital $i$ has $n$ features $x_i^{(p)} \in \R^n$, e.g.,~height, weight, gender, blood pressure, etc., and observations $y_i^{(p)} \in \R$, e.g.,~their A1C. For each hospital $i$, the patient population is described by a true underlying nominal distribution, denoted by $\P_i$, so that $(x_i^{(p)}, y_i^{(p)}) \sim \P_i$, and an empirical distribution, denoted by $\hat{\P}_i$, that is obtained from the data in $\D_i$. We assume that the features and observations are related via a common underlying parametric model $f$, so that $y_i^{(p)} = f(x_i^{(p)};  \theta^{\ast}_i)$, where $\theta^{\ast}_i \in \R^n$ are the model parameters. Then, given a loss function $\mathcal{L}$, the goal of each hospital is to estimate the true parameters of the model $\theta_i^\ast$ that minimize the loss on the underlying true data distribution. To do so, we formulate the \emph{empirical risk minimization (ERM)} problem  $$\inf_{\theta \in \R^n} \E_{(x_i,y_i) \sim \hat{\P}_i} \bigl[ \mathcal{L} \left( f(x_i;\theta), y_i \right) \bigr].$$

We consider a federated setting, where hospitals voluntarily collaborate with each other in order to train more accurate local models. To mitigate possible biases in the local models that can be caused by differences in the patient populations served by the hospitals, the hospitals form coalitions so that knowledge is shared among hospitals with common or similar data-generating distributions. These coalitions are coordinated by a lead hospital. Specifically, the lead hospital aims to create a \emph{coalition structure} $\pi:\{1,\dots,N\}\to \{S_1, \dots, S_K\}$, that is a map from hospitals to $K$ coalitions $\{S_1, \dots, S_K\}$ that satisfy $S_1 \cup \cdots \cup S_K = \{1,\dots,N\}$, $S_k\cap S_{k'}=\emptyset$ for all $k\neq k'$, and $S_k\neq \emptyset$ for all $k\in\{1\dots,K\}$. Then, $\pi(i)=S_k$ denotes the unique coalition $S_k$ to which hospital $i$ is assigned. The  hospitals participating in this coalition structure first use their limited data to learn parameter estimates $\theta_i^{\text{local}}$, which we collect in a tuple $\btheta = \bigl( \theta_1^{\text{local}}, \dots, \theta_N^{\text{local}} \bigr)$. Then, they share these parameters with the lead hospital. The lead hospital aggregates the local parameters for each coalition and returns their mean to the member hospitals. Thus, every hospital $i$ belonging to coalition $\pi(i) = S_k$ receives the common estimated parameters
\begin{align}\label{eq:agg_theta}
    \theta_{S_k} = \frac{1}{|S_k|} \textstyle\sum_{i \in S_k} \theta_i^{\text{local}}.
\end{align}
Using the model parameters $\theta_{S_k}$, the expected loss of hospital $i$ in coalition $\pi(i) = S_k$ becomes
\begin{align}
\ell_i(\btheta, \pi) = \E_{(x_i, y_i)\sim \hat{\P}_i}\bigl[ \mathcal{L}\bigl(f(x_i; \theta_{S_k}), y_i\bigr) \bigr].
\end{align}

The difficulty in solving the ERM problem defined before lies in the fact that the lead hospital requires knowledge of the empirical distributions $\hat{\P}_i$ for all hospitals. Even more, due to finite-sample bias or other distributional shifts, whether the empirical distributions $\hat{\P}_i$ are a good representation of the underlying data-generating method also comes into question. To model this uncertainty, for each hospital $i$, we construct an ambiguity set that is a Wasserstein ball of radius $\epsilon_i \geq 0$ around $\hat{\P}_i$, denoted by $\B_{\epsilon_i}(\hat{\P}_i)$; see Section \ref{sec:background}. The radius of the ambiguity set $\epsilon_i$ is selected by each hospital individually, and represents how robust and conservative this hospital wants to be when computing its loss.
We argue that the mechanism for dividing hospitals into coalitions implemented by the lead hospital should take into account these ambiguity sets.
Recall that if a hospital $i$ belongs to coalition $\pi(i)=S_k$, then it receives model parameters $\theta_{S_k}$. Using these model parameters, hospital $i$ can also compute an upper bound on its expected loss by calculating the worst-case loss over all the distributions in its ambiguity set, i.e.,
\begin{equation}
    \small
    \ell_i^{\text{rob}}(\btheta,\pi) = \textstyle\sup_{\Q_i \in B_{\epsilon_i}(\hat{\P}_i)} \E_{(x_i, y_i)\sim \Q_i}\bigl[ \mathcal{L}(f(x_i;\theta_{S_k}), y_i) \bigr].
\end{equation}
Therefore, to address the aforementioned challenges, the lead hospital should form clusters that minimize the accumulated \textit{worst possible} loss across all hospitals. We translate this objective to the following optimization problem.

\begin{problem}\label{prob:main_problem}
     Given $N$ hospitals, $K$ coalitions, a loss function $\mathcal{L}$, and an ambiguity set $B_{\epsilon_i}(\hat{\P}_i)$ for each hospital $i$, compute a coalition structure $\pi:\{1,\dots,N\}\to \{S_1, \dots, S_K\}$ that minimizes the expected robust loss of all hospitals, i.e.,
    \begin{align*}
        \begin{aligned}
            \min_{\pi} & \quad  \textstyle\sum_{k=1}^K  \textstyle\sum_{i \in S_k} \ell_i^{\text{rob}} (\btheta,\pi)  \\
            \text{s.t.} &\quad S_1 \cup S_2 \cup \dots S_K = \{1,\dots,N\} \\
            &\quad S_{k} \cap S_{k^{\prime}} = \emptyset, \quad \forall k \neq k^{\prime} \\
            &\quad S_{k} \neq \emptyset, \quad \forall k \in \{1,\dots,K\}.
        \end{aligned}
    \end{align*}
\end{problem}


\vspace{0.5cm}
\section{Coalition Formation} \label{sec:coalition}


In order to solve Problem \ref{prob:main_problem}, we introduce a binary variable $a_{i,k} \in \{0,1\}$  such that $a_{i,k} = 1$ if hospital $i$ is assigned to coalition $S_k$ and $a_{i,k} = 0$ otherwise. Then, the coalition formation Problem \ref{prob:main_problem} can be reformulated into an integer program (IP) with binary decision variables $\{a_{i,k}\}_{N\times K}$ as
\begin{align}
\begin{aligned}
    \min_{a_{i,k}}&\quad \textstyle\sum_{k=1}^K \textstyle\sum_{i=1}^N a_{i,k} \ell_i^{\mathrm{rob}}(\btheta,\pi) \\
    \text{s.t.} &\quad \textstyle\sum_{k=1}^K a_{i,k} = 1, \quad i = 1,\dots,N \\
   &\quad \textstyle\sum_{i=1}^N a_{i,k} \geq 1, \quad  k =1,\dots,K\\
    &\quad a_{i,k} \in \{0,1\}, \quad  i = 1,\dots,N ,~k = 1,\dots,K.
\end{aligned} \label{eq:integer}
\end{align}

\subsection{Linear Integer Relaxation}

In order to solve the optimization problem \eqref{eq:integer}, we need to compute the robust loss $\ell_i^{\mathrm{rob}}(\btheta,\pi)$ of hospital $i$ for all possible aggregated models of all possible coalition structures $\pi$. This is a combinatorial problem that is very hard to solve. To address this challenge, 
we assume that the loss function $\mathcal{L}$ is convex with respect to the model parameters $\btheta$. As shown in the following result, this assumption allows us to relax problem \eqref{eq:integer} into a linear integer program that can be efficiently solved.

\begin{lemma} \label{lemma:upper-bounds}
    Suppose a parametric model $f(x;\theta)$ defined by model parameters $\theta$ and suppose that $\mathcal{L}(f(x;\theta), y)$ is a loss function that is convex with respect to $\theta$. Suppose also binary variables $a_{i,k}  \in  \{0, 1\}$, for all $i \in \{1,\dots,N\},~k \in \{1,\dots,K\}$, with $\textstyle\sum_{k=1}^{K} a_{i,k} = 1$ for all $i \in \{1,\dots,N\}$, that denote whether hospital $i$ belongs to coalition $S_k$. 
    Then,
    \begin{align*}
        \begin{aligned}
        &\textstyle\ell_i^{rob} (\btheta, \pi) \leq\\ &\textstyle\sum_{j=1}^N \frac{a_{j,k}}{\textstyle\sum_{j=1}^N a_{j,k}} \sup_{\Q_i \in B_{\epsilon_i}(\hat{\P}_i)} \E_{(x,y)\sim \Q_i} \bigl[ \mathcal{L}(f(x;\theta_{j}), y ) \bigr].
        \end{aligned}
    \end{align*}
\end{lemma}

\vspace{0.3cm}

\begin{proof}
    See Appendix.
\end{proof}

\vspace{0.1cm}

Define the loss
\begin{equation}\label{eq: robust_loss}
    L_{i,j}(\btheta) = \textstyle\sup_{\Q_i \in B_{\epsilon_i}(\hat{\P}_i)} \E_{(x,y)\sim \Q_i} \bigl[ \mathcal{L}(f(x;\theta_{j}), y ) \bigr],
\end{equation}
which we interpret as the wort-case loss obtained when hospital $i$ evaluates the model parameters of hospital $j$ over the ambiguity set centered at the empirical distribution of hospital $i$. In other words, $L_{i,j}(\btheta)$ is the robust transfer loss. Then, we have the following result.

\begin{theorem}\label{thm:1}
As before, suppose a parametric model $f(x;\theta)$ defined by model parameters $\theta$ and suppose that $\mathcal{L}(f(x;\theta), y)$ is a loss function that is convex with respect to $\theta$. Then, for
the optimal value of problem \eqref{eq:integer} we have that it is upper bounded by: 
    \begin{align*}
        \begin{aligned}
      \textstyle\sum_{i,j,k=1}^{N,N,K} \frac{a_{i,k}\cdot a_{j,k}}{\textstyle\sum_{j = 1}^N a_{j,k}} L_{i,j}(\btheta).
        \end{aligned}
    \end{align*}
\end{theorem}

\vspace{0.3cm}

\begin{proof}
    Apply Lemma \ref{lemma:upper-bounds} directly to \eqref{eq:integer}.
\end{proof}

\vspace{0.1cm}

Replacing the objective function of the coalition formation problem \eqref{eq:integer} by the upper bound in Theorem \ref{thm:1} and introducing the binary variables $a_{i,j,k} \in \{0, 1\}$, for all $i \in \{1, \cdots, N\}$, $j \in \{1, \cdots, N\}$, $k \in \{1, \cdots, K\}$, so that $a_{i,j,k} = 1$ if $a_{ik} = 1 \text{ and } a_{jk} = 1$, the coalition formation problem \eqref{eq:integer} can be relaxed to the following linear integer fractional program:
\begin{align}
\begin{aligned}
    \min_{a_{i,k}}&\quad      \textstyle\sum_{i,j,k=1}^{N,N,K} \frac{a_{i,j,k}}{\textstyle\sum_{j = 1}^N a_{j,k}} L_{i,j}(\btheta) \\
    \text{s.t.} &\quad \textstyle\sum_{k=1}^K a_{i,k} = 1, \;\; i \in \{1..N\} \\
   &\quad \textstyle\sum_{i=1}^N a_{i,k} \geq 1, \;\;  k \in\{1..K\}\\
   &\quad a_{i,j,k} \leq a_{i,k},\;\; i, j \in \{1..N\} ,~k \in \{1..K\}\\
   &\quad a_{i,j,k} \leq a_{j,k},\;\; i, j \in \{1..N\} ,~k \in \{1..K\}\\
   &\quad a_{i,k} + a_{j,k} - a_{i,j,k} -1 \leq 0,\;\; i, j \!\in\! \{1..N\},k \!\in\! \{1..K\}\\
    &\quad a_{i,k} \in \{0,1\}, \;\;  i \in \{1..N\} ,~k \in \{1..K\}\\
    &\quad a_{i,j,k} \in \{0,1\}, \;\;  i, j \in \{1..N\} ,~k \in \{1..K\}.
\end{aligned} \label{eq:coal_final}
\end{align}

In what follows, we analyze how to compute the robust transfer loss in \eqref{eq: robust_loss} for two specific model classes: linear models with $\ell_1$-loss and logistic regression.

\subsection{Example: $\ell_1$-linear regression}

In the case of linear models, the robust transfer loss in \eqref{eq: robust_loss} can be written as
\begin{align}
    L_{i,j}(\btheta) = \textstyle\sup_{\Q_i \in B_{\epsilon_i}(\hat{\P}_i)} \E_{(x,y)\sim \Q_i} \bigl[ |\theta_{j}^{\top} x - y | \bigr]. \label{eq:transfer-linear}
\end{align}

To compute this loss, we can use the dual formulation for linear regression with $\ell_1$-loss given in Lemma \ref{lem:linear}. More specifically, for fixed model parameters $\btheta$ we get
\begin{align}
\label{eq:proposition-epsilon}
    L_{i,j}(\btheta) = \epsilon_i \cdot \alpha + \E_{(x,y)\sim \hat{\P}_i}\bigl[ |\theta_{j}^{\top} x - y | \bigr],
\end{align}
where $\alpha$ is a decision variable in the dual formulation; see Eq.~\eqref{eq:lemma-2}. Substituting the robust losses of every model $\theta_j$ and every hospital $i$ into the coalition formation problem \eqref{eq:coal_final}, we can obtain the desired optimal coalition structure.

Finally, recall that every hospital is responsible for selecting the value for the radius of its ambiguity set $\epsilon_i$. If all hospitals select the same radius for their ambiguity sets, the following result holds true.
\begin{proposition}
    Assume all hospitals use the same value $\epsilon$
 for the radius of their ambiguity sets. Then, the coalition structure returned by the solution of the coalition formation problem \eqref{eq:coal_final} is independent of $\epsilon$.

    \begin{proof}
        See Appendix
    \end{proof}
\end{proposition}

\subsection{Example: Logistic Regression}

In the case of logistic regression models, the robust transfer loss in \eqref{eq: robust_loss} can be written as
\begin{align}
    L_{i,j}(\btheta) \!=\! \textstyle\sup_{\Q_i \in B_{\epsilon_i}(\hat{\P}_i)} \E_{(x,y)\sim \Q_i} \bigl[ \log(1 + \exp(-y \cdot \theta_j^Tx)) \bigr].\label{eq:transfer-logistic}
\end{align}

As in the case of linear models discussed above, here too we can use the dual formulation for logistic regression given in Lemma \ref{lem:logistic} to compute the robust loss in \eqref{eq:transfer-logistic}. More specifically, for fixed model parameters $\btheta$ we get
\begin{align}
    L_{i,j}(\btheta) = \epsilon_i \cdot \alpha + \frac{1}{N}\textstyle\sum_{i=1}^N b_i,
\end{align}
where both $\alpha$ and $b_i$ are decision variables in the dual formulation; see Eq.~\eqref{eq:lemma-3}. Like in the linear case, substituting the robust losses of every model $\theta_j$ and every hospital $i$, into the coalition formation problem \eqref{eq:coal_final},
we can obtain the desired optimal coalition structure.

\section{Federated Learning Protocol} \label{sec:protocol}

\begin{figure*}[h]
\centering
\begin{subfigure}[b]{0.45\textwidth}
    \includegraphics[width=1\textwidth]{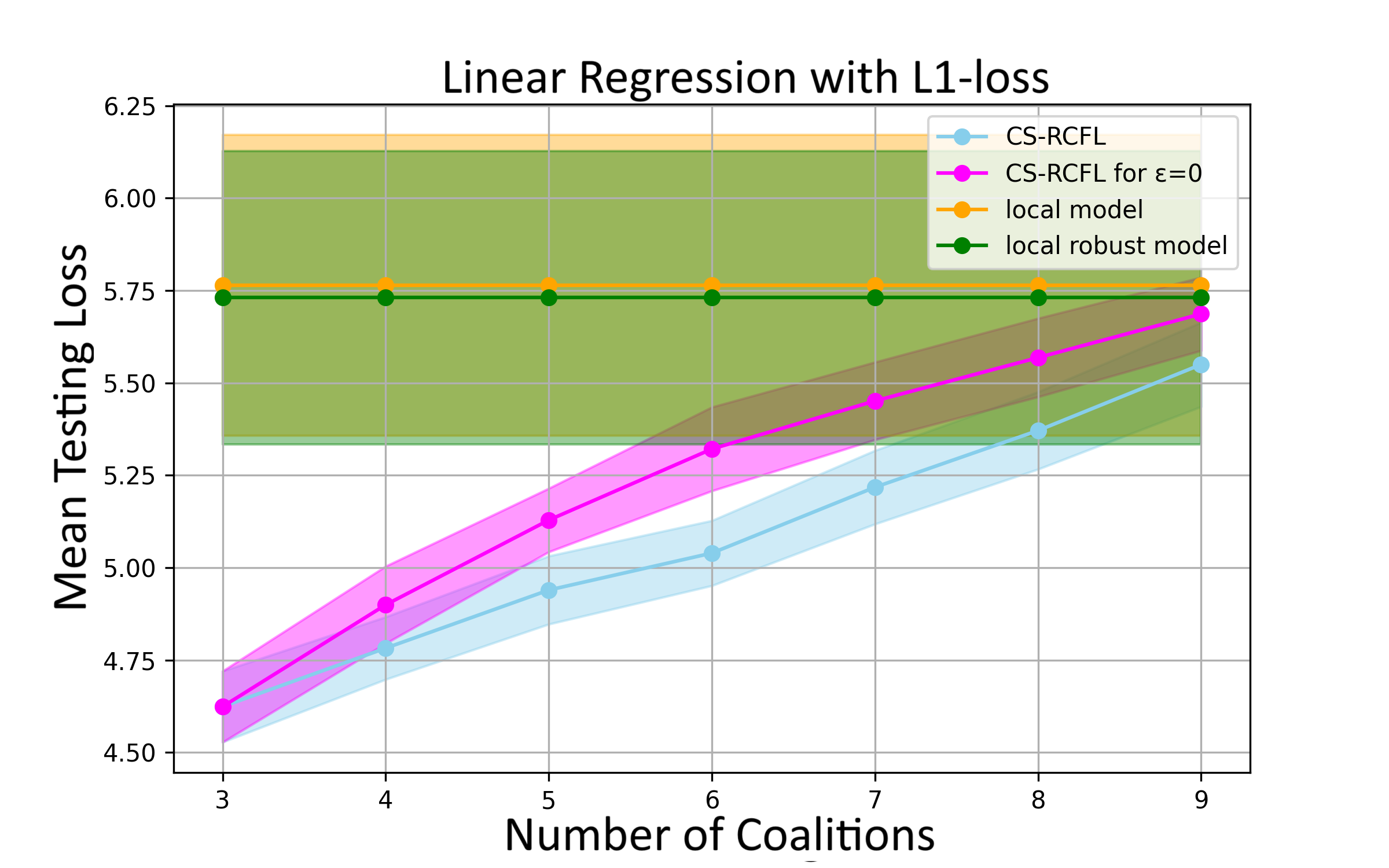}
    \caption{}
    \label{fig:left_plot}
\end{subfigure}%
\begin{subfigure}[b]{0.45\textwidth}
    \includegraphics[width=1\textwidth]{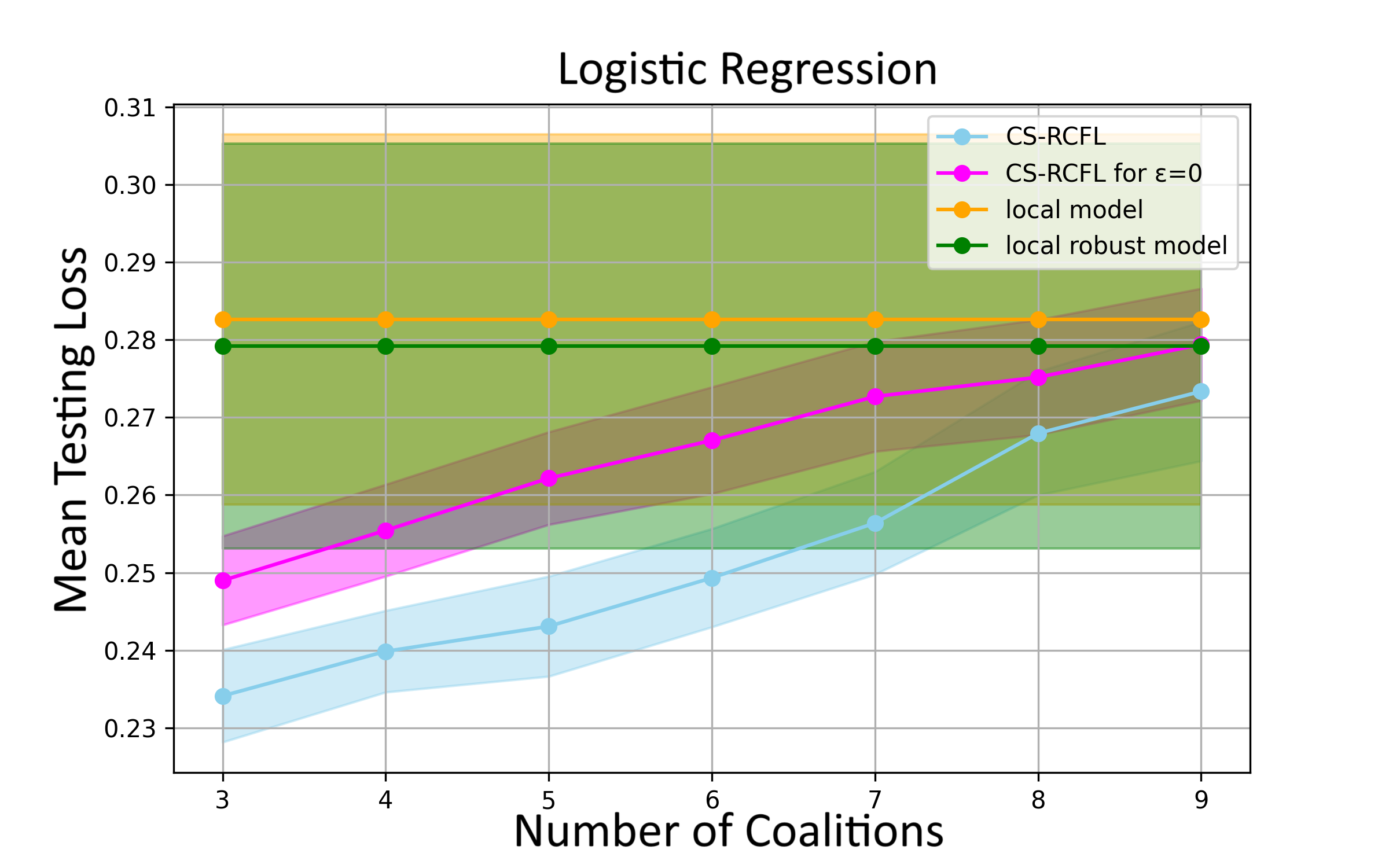}
    \caption{}
    \label{fig:right_plot}
\end{subfigure}%
\caption{\footnotesize Loss of the CS-RCFL method and the benchmarks for (a) linear regression models with absolute ($\ell_1$-) loss and (b) logistic regression models.}
\label{fig:combined_results}
\end{figure*}

In this section, we discuss a protocol that describes how the proposed clustered federated learning (CFL) method can be implemented. 

\begin{enumerate}
    \item \emph{Training Phase}: Initially, the hospitals use their local data to learn estimates of their local model parameters $\theta_i^{\text{local}}$. Then, they share their estimated local model parameters with the lead hospital and the coalition formation phase begins.
    \item \emph{Communication Phase}: To solve the coalition formation problem  \eqref{eq:coal_final}, the lead hospital requires a conservative estimate of the loss of every local model computed on the empirical distribution of every hospital, which is the robust transfer loss defined in \eqref{eq: robust_loss}. To obtain these robust losses, the lead hospital sends to all participating hospitals the local model parameters it has received, without revealing to which hospital each parameter set belongs to; this way information privacy is maintained as no hospital can determine which hospital trained which model.
    When the hospitals have received the model parameters of all other hospitals, they compute for each model the worst expected loss over all patient distributions in their local ambiguity set and send these values to the lead hospital.
    \item \emph{Coalition Phase}: When the lead hospital has received the worst-case expected losses of all models from all hospitals, it can compute the optimal coalition structure by solving the coalition formation problem \eqref{eq:coal_final}. Then, the lead hospital can compute the aggregate model parameters for each coalition by \eqref{eq:agg_theta} and send these aggregate parameters to the member hospitals of each coalition.
\end{enumerate}

\section{Experimental Results} \label{sec:experiments}

\subsection{Synthetic Data}

We validate the effectiveness of the proposed CS-RCFL algorithm for linear regression models with absolute ($\ell_1$) loss and for logistic regression models.
We compare the performance of CS-RCFL to three benchmarks:
\begin{itemize}
    \item The local models that the hospitals can learn from their local data, without considering distribution shifts or collaboration with other hospitals.
    \item The robust local models that the hospitals can learn by minimizing their robust local loss without collaborating with other hospitals.
    \item The non-robust clustered federated learning models that the hospitals can learn, without considering distribution shifts, i.e.,~using the non-robust losses to perform the clustering. This benchmark coincides with the proposed CS-RCFL algorithm for $\epsilon = 0$.
\end{itemize}

Below we describe the synthetic datasets we generate for the linear regression and logistic regression models. These datasets are different because the first task is a regression problem whereas the second one is a classification problem.
\begin{itemize}
    \item \emph{Linear Regression Dataset.} The dataset consists of $N=10$ hospitals. The number of samples $n$ of each hospital ranges between $50$ and $150$, emulating hospitals with various sizes of patient populations. Every data sample has $50$ features. For each hospital $i$, we generate the observations as $y_i^{(p)} = \tilde{w}_i^{\top} x_i^{(p)} + \eta$, where $x_i \sim \mathcal{N}(\boldsymbol{\mu}_i,\Sigma_i)$ is sampled from a multi-variate normal distribution with random covariance matrix and mean, $\eta \sim \mathcal{N}(0,\sigma^2)$ is an i.i.d.~noise with standard-deviation $\sigma = 5$, and $\tilde{w}_i$ is the true weight of the model. To simulate hospitals with similar joint distributions (that can be clustered together), we select $3$ different values for $\tilde{w}$ and assign each hospital to one of the $3$ true weights.
    \item \emph{Logistic Regression Dataset.} Like before, this dataset consists of $N=10$ hospitals and the number of samples $n$ of each hospital ranges between 100 and 200. Every data sample has $50$ features. For each hospital $i$, we generate the observations by the (log-loss) probability $P(y_i^{(p)} = 1|x_i^{(p)}) = [1 + \exp(-\tilde{w}_i^{\top} x_i^{(p)})]^{-1}$. We generate labels based on a threshold on the log-loss; if $P(y_i^{(p)} = 1|x_i^{(p)}) > 0.5$, then we set $y^{(p)} = 1$, otherwise $y^{(p)} =  -1$.
\end{itemize}



In our proposed CS-RCFL method, we select the radii of the ambiguity sets of the hospitals depending on the number of samples each one of them has. If $n > 100$, we set $\epsilon = 1$ for linear models and $\epsilon = 0.5$ for logistic regression models. Otherwise, we use $\epsilon = 2$ for  linear models and $\epsilon = 1$ for logistic regression models. For the robust local models method, we use cross-validation to select the radii of the ambiguity sets. Specifically, for linear models we get the values of $\epsilon$ that we also use in the CS-RCFL algorithm, while for logistic regression models we get $\epsilon = 0.05$ if $n>100$ and $\epsilon = 0.1$, otherwise. Note that cross-validation is hard to implement in practice for the CS-RCFL algorithm as this would require coordination with the lead hospital. Instead, we selected conservative ambiguity sets (e.g., for logistic regression models) so that the  coalitions returned by CS-RCFL are robust to possibly large distribution shifts. 


In what follows, we examine the performance of our algorithm for different numbers of coalitions $K \in \{3, 4,\dots, 10\}$. Specifically, we compare our proposed CS-RCFL method and the benchmarks described above on 10 testing datasets that contain 100 data-samples each, generated from the same distributions as the training data for each one of the $N=10$ hospitals. In Fig.~\ref{fig:combined_results} we report the loss of each method averaged over the 10 different test datasets. We observe that even without federated learning, the use of DRO slightly improves the performance of the local models, which is expected due to the presence of distribution shifts between the training and test data. On the other hand, our CS-RCFL method consistently outperforms the other benchmarks by achieving lower average model loss and lower variance. It is also worth emphasizing that the models returned by CS-RCFL have lower loss compared to those returned by the clustered federated learning method without robustness, showing that our method is more effective in addressing distribution shifts. Moreover, this improved performance is consistent for all values of $K$, meaning that our algorithm is robust to the choice of the numbers of coalitions, which is typically unknown. We finally note that our algorithm is able to consistently recover the true coalition structure in the case $K=3$.

\subsection{Real-world Healthcare Data}

\begin{figure*}[h]
\centering
\begin{subfigure}[b]{0.45\textwidth}
    \includegraphics[width=0.8\textwidth]{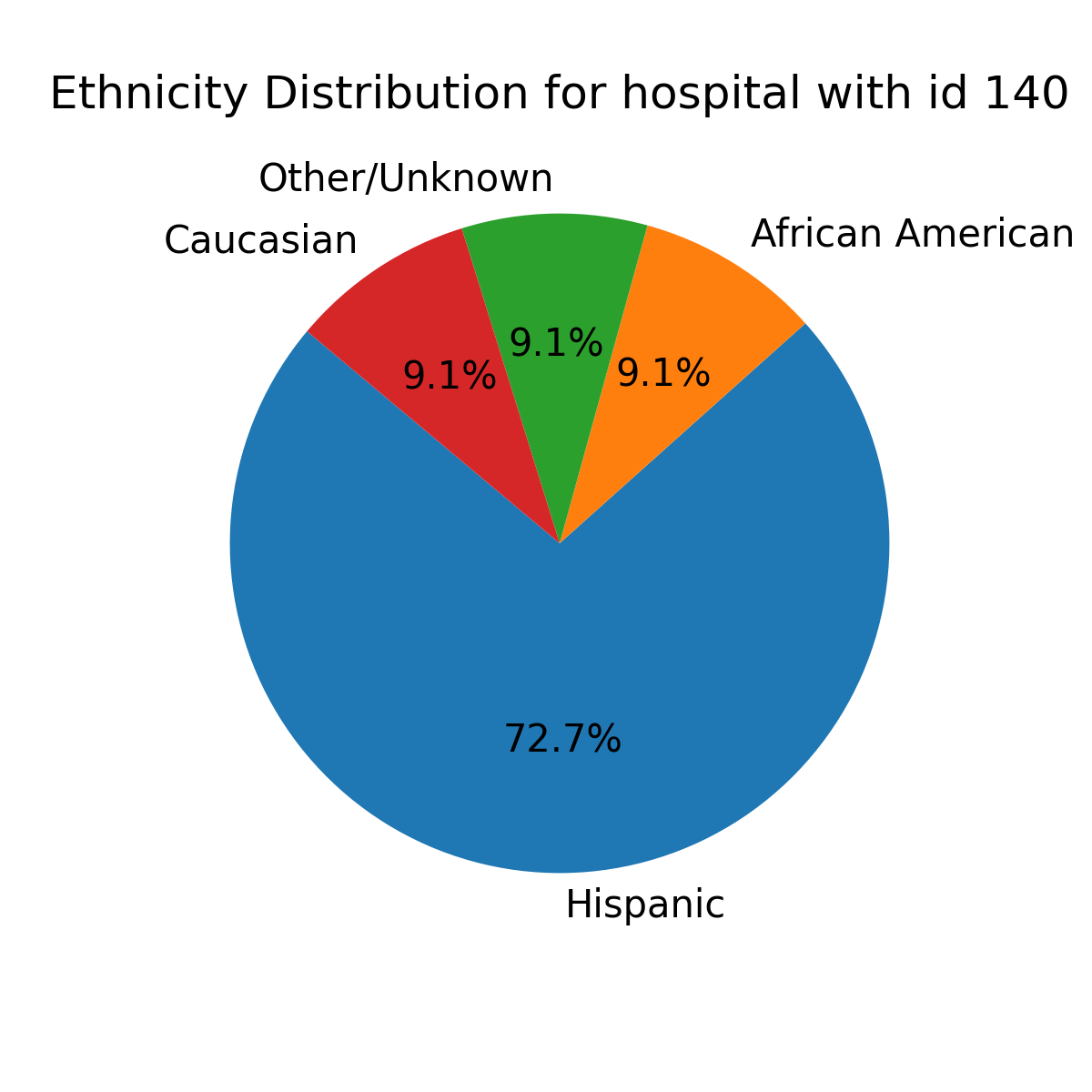}
    \label{fig:ethn_140}
\end{subfigure}%
\begin{subfigure}[b]{0.45\textwidth}
    \includegraphics[width=0.8\textwidth]{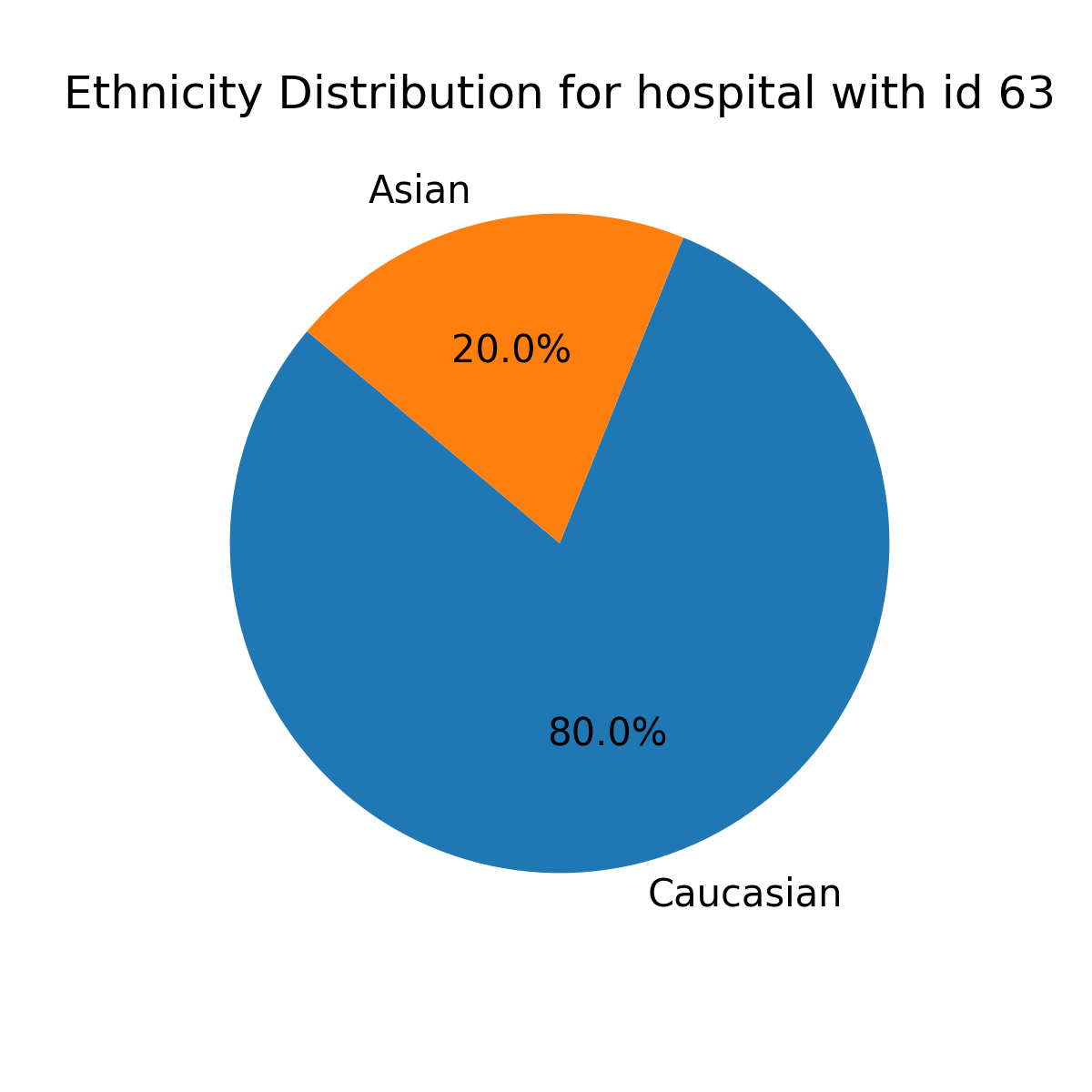}
    \label{fig:ethn_63}
\end{subfigure}%
\vspace{-30pt}
\caption{\footnotesize Distribution of patients' ethnicity at two different hospitals, an example that shows heterogeneity of patient populations across hospitals.}
\label{fig:patient_population}
\end{figure*}




\begin{figure}[h]
    \centering
    \includegraphics[width=1\linewidth]{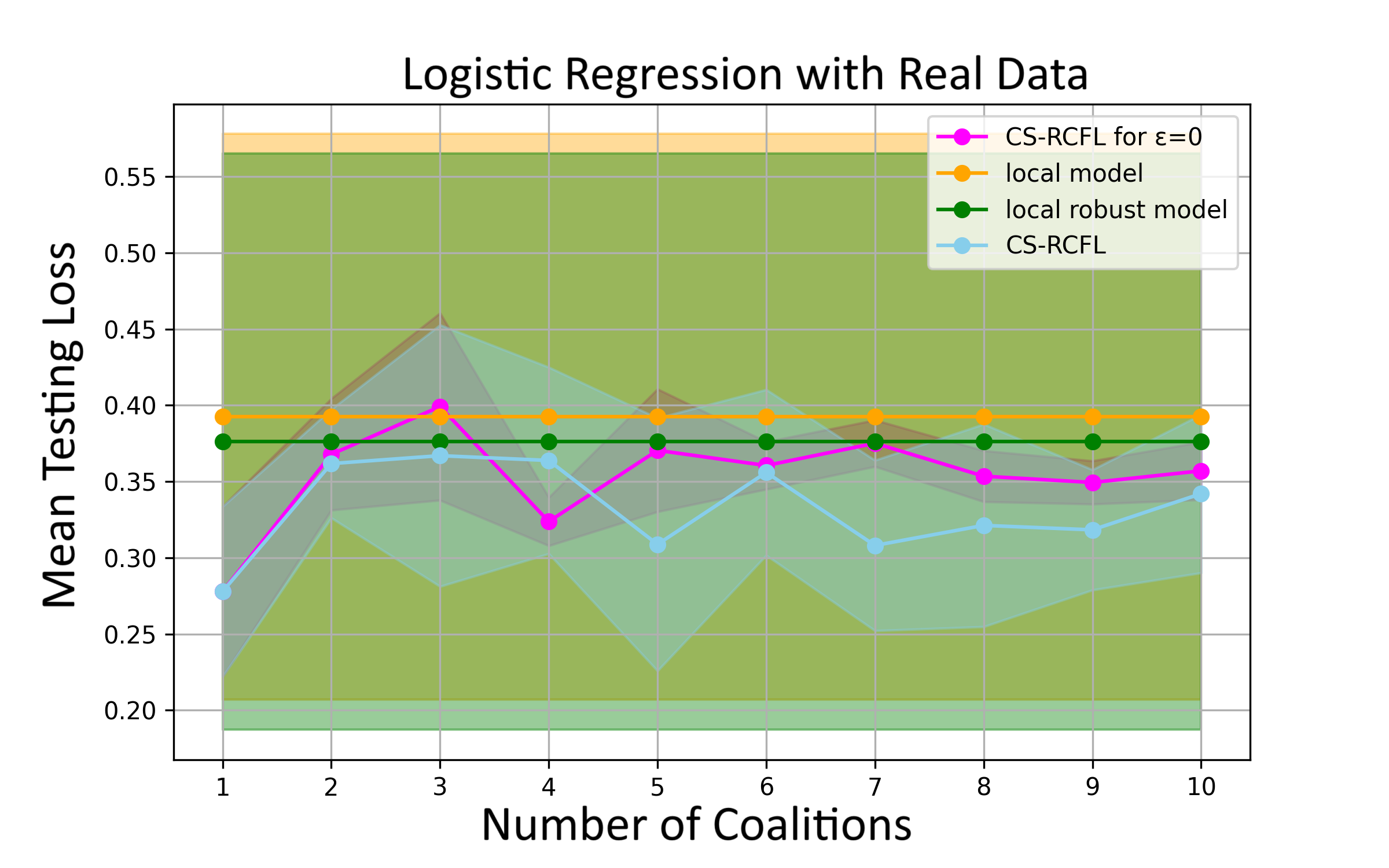}
    \caption{\footnotesize{Loss of the CS-RCFL method and the benchmarks for the logistic regression model, evaluated on the eICU Collaborative Research Dataset.}}
    \label{fig:real_log_results}
\end{figure}

Next, we validate our algorithm on real-world healthcare data. We use the open access demo version of the eICU Collaborative Research Database, a multi-center database comprised of de-identified health data of over 200,000 admissions to ICUs across the United States between 2014-2015\cite{pollard2018eicu}. Since not all patients in this dataset have the same number of attributes, we select a subset of attributes and construct a dataset with the patients that have those attributes in common. This way, we end up with data from 174 different hospitals with each hospital having between 3 and 18 patients admitted to ICU. For each one of these patients we have access to 57 different features, including demographic information (e.g., age, sex, ethnicity), vital sign measurements, severity of illness measures, and other diagnosis information (e.g., temperature, heart-rate, apachescore etc.). As seen in Fig. \ref{fig:patient_population}, the dataset consists of heterogeneous hospitals that serve different patient populations. 

An important feature contained in the dataset for every patient is their length-of-stay in the ICU. Knowledge of the length-of-stay in the ICU (similarly, knowledge of the utilization of other hospital resources, such as PACUs or step-down beds) can be used to streamline health system operations  and improve delivery of care. Therefore, in this experiment, we focus on ICU length-of-stay prediction. Specifically, to simplify the problem, we consider a logistic regression model that can predict whether a patient will stay in the ICU longer than 1 day or not. Since the number of data samples at each hospital is very low, especially as it relates to the large number of patient features, we consider hospitals that have at least 10 data samples each. Among those, we randomly select 20 hospitals to participate in the proposed federation. Moreover, for these 20 hospitals, we reduce the number of features we use to train the logistic regression models from 57 to 6. These 6 features include both demographic information (gender, ethnicity, and age) and illness measurements (medicines, heart-rate, and apachescore).

To decide the radius $\epsilon$ of the ambiguity set of each hospital we use cross-validation. We get values of $\epsilon$ that range between $[0.001, 0.1]$ across the different hospitals. Small values (i.e., $\epsilon=0.001$) indicate that the training and validation sets are similar and thus there are small distribution shifts in the data. On the other hand, larger values (i.e., $\epsilon=0.1$) mean that the training distribution is significantly different compared to the validation one. To validate our method, we split the data at every hospital into a training set containing 70\% of the data and a test set containing the rest. We then train our CS-RCFL model along with the other benchmarks on the training data and evaluate their loss on the corresponding test sets at the local hospitals. We construct 3 different experiments by randomly selecting 3 different training sets at each hospital and, in Fig. \ref{fig:real_log_results}, we report the average loss of all models (averaged over the 3 experiments) for various coalition number values $K \in \{1,\dots, 10\}$. We observe that the fact that hospitals have very few data affects the performance of the local models; both the local and local robust models suffer from high loss and high variance. Moreover, we observe that our proposed CS-RCFL model (with either $\epsilon=0$ or $\epsilon>0$) outperforms the models trained locally at each hospital both in average loss and in terms of variance. Finally, we see that incorporating robustness, i.e., letting $\epsilon>0$, can on average improve the loss of a clustered federated model.

\section{Conclusion} \label{sec:conclusion}

In this paper, we proposed a new clustered federated learning method that assigns hospitals to coalitions allowing hospitals in the same coalition to collaboratively train a common model. We assumed that the local data at each hospital may be subject to local distribution shifts and may also be statistically different across hospitals. Our proposed clustered federated learning method designs coalitions that are robust to distribution shifts in the local data and learns local models that are unbiased in the presence of statistically heterogeneous hospitals. We evaluated our method on synthetic and real healthcare data and showed that it outperforms models that are trained solely on local data or federated models that are not robust to distribution shifts. In future work, we will explore the use of a broader class of models, including multi-layer perceptrons, Neural Networks, etc.

\appendix

\begin{proof}[Proof of Lemma \ref{lemma:upper-bounds}]
Let $a_k = \textstyle\sum_{i=1}^N a_{i,k}$, and let $\mathcal{P}_i = B_{\epsilon_i}(\hat{\P}_i)$. Applying Jensen's inequality, along with the monotonicity of the supremum and the additive property of expectation yields $\ell_i^{rob}(\btheta, \pi) =$
\begin{align*}
\sup_{\Q_i \in \mathcal{P}_i} \E_{(x,y) \sim \Q_i}\left[ \mathcal{L}\left(\textstyle\sum_{j=1}^N \frac{a_{j,k}}{a_k}\theta_{j}^\top x, y \right)\right] \leq \\
    \sup_{\Q_i \in \mathcal{P}_i} \E_{(x,y) \sim \Q_i}\left[ \textstyle\sum_{j=1}^N \frac{a_{j,k}}{a_k} \mathcal{L} \left( \theta_j^{\top} x, y \right)\right]  = \\
    \sup_{\Q_i \in \mathcal{P}_i}  \textstyle\sum_{j=1}^N \frac{a_{j,k}}{a_k} \E_{(x,y) \sim \Q_i}\left[\mathcal{L} \left( \theta_j^{\top} x, y \right)\right] = \\
       \textstyle\sum_{j=1}^N \frac{a_{j,k}}{a_k} \sup_{\Q_i \in \mathcal{P}_i} \E_{(x,y) \sim \Q_i}\left[\mathcal{L} \left( \theta_j^{\top} x, y \right)\right].
\end{align*}
\end{proof}

\begin{proof}[Proof of Proposition 1]
From \eqref{eq:proposition-epsilon} and $\alpha = \sqrt{||\theta_i||_2^2 + 1}$ from Lemma \ref{lem:linear}, if $\epsilon_i = \epsilon$ $\forall i \in \{1,\dots,N\}$, then the objective becomes $\sum_{i=1}^N \epsilon \sqrt{||\theta_i||_2^2 + 1} + $
    \begin{align*}
         &\textstyle\sum_{k=1}^K \textstyle\sum_{i=1}^N \textstyle\sum_{j=1}^N \frac{a_{i,j, k}}{\textstyle\sum_{j = 1}^N a_{j,k}} \bigl( \E_{(x,y)\sim \hat{\P}_i}\bigl[ |\theta_{j}^{\top} x - y | \bigr]\bigr).
    \end{align*}
In this case, the coalition structure returned by this method is independent of the value of the radius of the ambiguity set $\epsilon$, since all the decision variables $a_{i,j,k}, a_{j,k}$ are independent of $\epsilon$.
\end{proof}
\addtolength{\textheight}{-11cm}
\bibliographystyle{ieeetr}
\bibliography{IEEEabrv,bibliography}
\end{document}